\documentclass[letterpaper, 10 pt, conference]{ieeeconf}  

\IEEEoverridecommandlockouts                              
\usepackage{bmathshortcuts}

\newtheorem{theorem}{Theorem}

\usepackage{algorithmic}
\usepackage[ruled,vlined]{algorithm2e}
\usepackage{amsmath,amssymb,amsfonts}
\usepackage{bm}
\usepackage{booktabs}
\usepackage{color}
\usepackage{comment}
\usepackage{fancyhdr}
\usepackage{float}
\usepackage{fullpage}
\usepackage{gensymb}
\usepackage{graphicx}
\usepackage{caption} 
\usepackage{subcaption} 
\usepackage[hidelinks]{hyperref} 
\usepackage{mathrsfs}
\usepackage{soul}
\usepackage{textcomp}
\usepackage{times}
\usepackage[usenames,dvipsnames]{xcolor} 
\usepackage{url}
{\vspace{-\topsep}\begin{itemize}\itemsep1pt \parskip0pt \parsep1pt}
{\end{itemize}\vspace{-\topsep}}

\newenvironment{tightenumerate} 
{\vspace{-\topsep}\begin{enumerate}\itemsep1pt \parskip0pt \parsep1pt}
{\end{enumerate}\vspace{-\topsep}}

\begin{document}

\title{\LARGE \bf
Safe Balancing Control of a Soft Legged Robot}

\author{Ran Jing$^{1}$, Meredith L. Anderson$^{1}$, Miguel Ianus-Valdivia$^{1}$, Amsal Akber Ali$^{1}$, Carmel Majidi$^{2}$,\\ Andrew P. Sabelhaus$^{1,3}$
\thanks{This work was in part supported by the Office of Naval Research under Grant No. N000141712063 (PM: Dr. Tom McKenna), the National Oceanographic Partnership Program (NOPP) under Grant No. N000141812843 (PM: Dr. Reginald Beach), and an Intelligence Community Postdoctoral Research Fellowship through the Oak Ridge Institute for Science and Education.}
\thanks{$^1$R. Jing, M.L. Anderson, M. Ianus-Valdivia, A. Akber Ali and A.P. Sabelhaus are with the Department of Mechanical Engineering, Boston University, Boston MA, USA. {\tt\small \{rjing, merland, maiv, aakber, asabelha\}@bu.edu} }
\thanks{$^2$C. Majidi is with the Department of Mechanical Engineering and the Robotics Institute, Carnegie Mellon University, Pittsburgh PA, USA. {\tt\small cmajidi@andrew.cmu.edu}}
\thanks{$^3$A.P. Sabelhaus is also with the Division of Systems Engineering, Boston University, Boston MA, USA.}
}


\maketitle
\pagestyle{empty}  
\thispagestyle{empty} 

\begin{abstract}

Legged robots constructed from soft materials are commonly claimed to demonstrate safer, more robust environmental interactions than their rigid counterparts.
However, this motivating feature of soft robots requires more rigorous development for comparison to rigid locomotion.
This article presents a soft legged robot platform, Horton, and a feedback control system with safety guarantees on some aspects of its operation. 
The robot is constructed using a series of soft limbs, actuated by thermal shape memory alloy (SMA) wire muscles, with sensors for its position and its actuator temperatures.
A supervisory control scheme maintains safe actuator states during the operation of a separate controller for the robot's pose.
Experiments demonstrate that Horton can lift its leg and maintain a balancing stance, a precursor to locomotion.
The supervisor is verified in hardware via a human interaction test during balancing, keeping all SMA muscles below a temperature threshold.
This work represents the first demonstration of a safety-verified feedback system on any soft legged robot.
\end{abstract}












\section{Introduction}
Locomotion of robots built from soft and compliant materials has the potential to expand exploration efforts in extreme, delicate, or dangerous locations that are unsafe or difficult to reach for either humans or traditional rigid robots \cite{Rus_Tolley_2015,Calisti_fundamentals_2017}. 
However, the source of these robots' benefits -- conforming to unstructured environments and unanticipated disturbances -- introduces significant challenges in control, modeling, and mechanical design \cite{Rich_untethered_2018}.
Most soft robot locomotion has been limited to simple motions in open-loop \cite{scott_geometric_2020,Huang2019,patterson_untethered_2020,bern_trajectory_2019}, a stark contrast to the intelligent full-body control of rigid systems \cite{ames_rapidly_2014}, which often come with provable performance properties.


\begin{figure}[t]
    \centering
    \includegraphics[width=1\columnwidth]{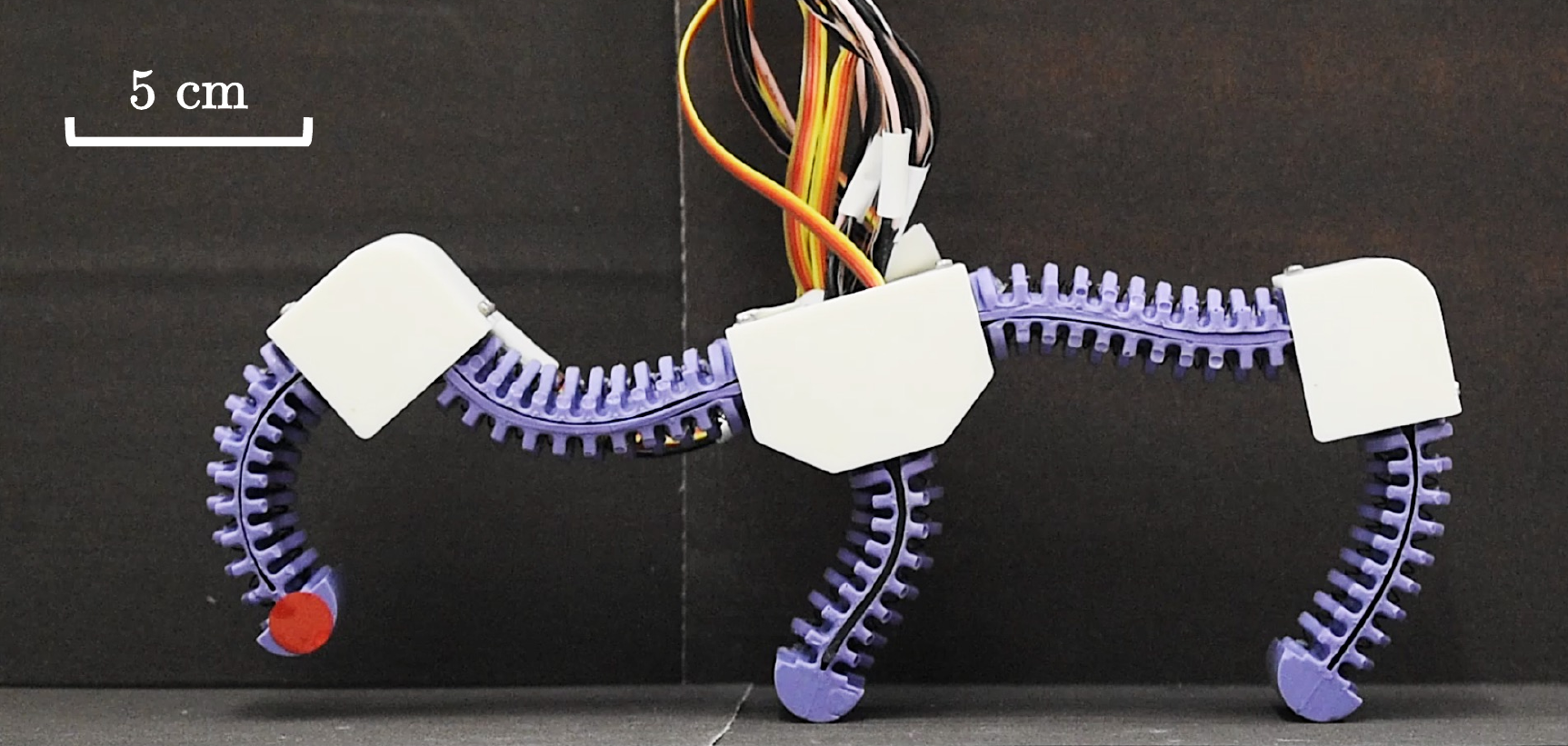}
    \caption{The soft legged robot Horton demonstrates dynamic balancing motions, and full pose control, using feedback with some formal safety verification.}
    \label{fig:horton_balancing}
    \vspace{-0.3cm}
\end{figure}

This article makes progress in bridging the gap.
We introduce a soft legged robot, Horton, with sufficient actuation, sensing, and feedback to demonstrate dynamic balancing (Fig. \ref{fig:horton_balancing}).
Horton's feedback control framework includes a safety verification for its most failure-prone component, the shape memory alloy (SMA) artificial muscles.
This is the first framework and experimental validation of full pose control with any verifiable safety for any soft legged robot, using dynamic balancing as a precursor to future walking locomotion \cite{ott_posture_2011}.



\begin{figure*}[t]
    \centering
    \includegraphics[width = \textwidth]{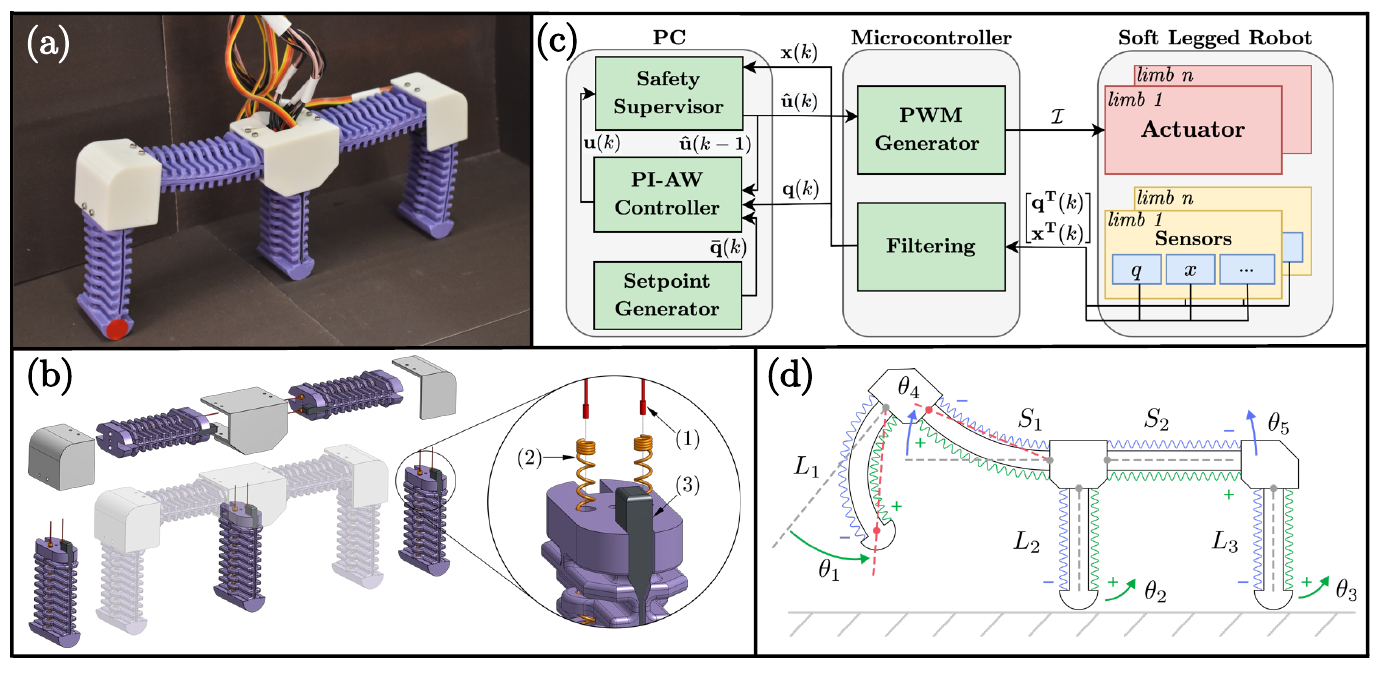}
    \caption{The Horton robot platform contains five SMA-actuated limbs arranged for planar motions (a), connected via 3D-printed brackets (b) with temperature sensors (b)(1) connected to the two SMAs per-limb (b)(2) alongside sensors for bending angle (b)(3). The robot's feedback system operates on an adjacent computer and microcontroller (c), which controls the angles $\theta_1 \hdots \theta_5$ of each of the five limbs (d). The ten corresponding SMAs are labeled according to leg/spine segment ($L_i/S_i$) and direction of induced bending ($+/-$).}
    \label{fig:horton_hardware_and_system_architecture}
    \vspace{-0.5cm}
\end{figure*}

\subsection{Robot Locomotion and Control, Soft or Rigid}

Mobile soft robots rely on the elastic deformation of their bodies for locomotion \cite{Calisti_fundamentals_2017}, often preventing the use of full state feedback while also limiting versatility.
Friction-based gaits such as crawling \cite{Tolley_quadrupedal_2014,Schiller_gecko_2020,Onal_snake_2013,Thomas_inchworm_2021} rely on continuous contact with the ground for stability, resulting in slow locomotion and inability to overcome obstacles of larger sizes.
Jumping motions are inherently limited in the accuracy of their landing location \cite{tolley_untethered_2014,Mintchev2018,Huang2019}.
Progress toward controlled soft locomotion has therefore focused on legged robots, but prior work has either been restricted to aquatic environments for stability \cite{Cianchetti_OCTOPUS_2015,Calisti_Poseidrone_2015,patterson_untethered_2020} or been limited with speed and unverified feedback approaches \cite{Li_modular_2022}.


In contrast, rigid legged robots can typically locomote over large environmental obstacles, in structured or semi-structured environments \cite{lee_learning_2020}.
Since dynamic balancing and locomotion are unstable, feedback control over the robot's entire pose is required \cite{majumdar_control_2014}.
Bringing these capabilities to soft robots will require new architectures that can execute similar control systems.

\subsection{Control and Safety}

State-of-the-art locomotion of rigid robots has been enabled by feedback control with verifiable properties, such as stability of hybrid systems \cite{sreenath_compliant_2011} and invariance \cite{ames_rapidly_2014}.
Set invariance formally defines safety: if a robot's states remain within some set for all time under the action of feedback, its operation is safe \cite{gilbert_linear_1991,blanchini_set_1999}.
In contrast, stable soft robot controllers have only been shown for manipulation \cite{della_santina_model-based_2020,della_santina_control_survey}.



Since balancing and locomotion involve environmental interaction, disturbances and unmodeled contact are major failure modes.
Soft robots, though informally `safe' by way of limiting their forces or positions \cite{lasota_survey_2017,tonietti_design_2005}, commonly use novel actuators \cite{zhang_robotic_2019} with possible catastrophic failure modes in these situations.
Under a strong disturbance, Horton's thermoelectric SMA actuators would overheat \cite{soother_challenges_2020} and fail, as would equivalent designs using pneumatics (bursting \cite{terrynSelfhealingSoftPneumatic2017}) and dielectrics (electrical breakdown or arcing \cite{planteLargescaleFailureModes2006,bilodeauSelfHealingDamageResilience2017}). 
As of yet, approaches for safer operation of soft robots of any type have relied on simple hard stops \cite{villanuevaBiomimeticRoboticJellyfish2011,coloradoBiomechanicsSmartWings2012,whiteSoftParallelKinematic2018,yeeharntehArchitectureFastAccurate2008}, with only limited examples of intelligent shutoff \cite{balasubramanian_fault_2020}.




\vspace{-0.1cm} 
\subsection{Contribution}



The Horton architecture (Fig. \ref{fig:horton_hardware_and_system_architecture}) contributes a platform particularly well-suited to studying balancing control and safety in soft robots.
Horton uses SMA-powered limbs to create high forces, strains, and speeds due to their work density \cite{Rich_untethered_2018}, extending prior work in soft limb design \cite{walters2011digital,wertz_trajectory_2022,patterson_robust_2022,patterson_untethered_2020,sabelhaus_-situ_2022}.
Consequently, Horton has dynamic motions and the associated instabilities that motivate control research in rigid legged robots.
This work applies a supervisory controller to Horton's SMAs during pose control, adapted for multiple soft limbs from prior work \cite{sabelhaus_safe_2022}.
Specifically, this article contributes:

\vspace{0.1cm}
\begin{tightenumerate}
    \item A soft robot legged robot platform with sufficient state sensing and actuation for full pose control,
    \item The first application of a verifiably-safe controller to a soft legged robot, and
    \item A demonstration of safe balancing motions during an otherwise-unsafe disturbance test.
\end{tightenumerate}


\vspace{0.1cm}
\section{Robot Architecture}

Horton is constructed from five soft limbs, arranged to create planar motions, with pairs of antagonistically-arranged SMA wire coils (Fig. \ref{fig:horton_hardware_and_system_architecture}(a),(b)) that produce bidirectional bending (Fig. \ref{fig:horton_hardware_and_system_architecture}(d)).
This simple planar design facilitates the development of fundamental control results for soft legged robots.
The robot is also connected to a tether for power and communication (Fig. \ref{fig:horton_hardware_and_system_architecture}(a)), oriented upward from the center of the body, with loose cables.
Though the tether may exert some external forces on the robot, those forces were not sufficient to achieve dynamically stable balancing -- i.e., the robot was often observed to fall during initial experiments.



\subsection{Mechanical Design, Actuation, and Sensing}

The soft limbs used in Horton are adapted from a prior design \cite{sabelhaus_-situ_2022}, now with 3D-printed brackets to connect them.
The molded body of the limb (Smooth-Sil 945, Smooth-On) houses the two SMA springs (Dynalloy Flexinol $90^\circ$C, 0.020'' diam.) in slotted compartments, allowing rapid convective cooling. 
Attached to the end of each SMA is a thermocouple (5TC-TT-K-36-72, Omega), affixed using thermally conductive, electrically insulating epoxy. 
A soft capacitive angular displacement sensor (Bendlabs 1-axis) is inserted along a slot in the side of the limb.

The SMAs are powered individually by MOSFETs connected to a 7V power supply, controlled by pulse width modulation voltage signals $\bu(k) \in [0,1]$ from a microcontroller. 
Pulling the MOSFETs high creates a current through the SMA, causing contraction due to Joule heating. 
The microcontroller sends temperature and bend angle measurements to a computer, which in turn specifies PWM duty cycle (Fig. \ref{fig:horton_hardware_and_system_architecture}(c)).



This article takes the robot body's state space as the set of angles $\bq = \begin{bmatrix}\theta_1 & \hdots & \theta_5\end{bmatrix}^\top$, Fig. \ref{fig:horton_hardware_and_system_architecture}(d).
Prior work has shown how tip-tangent angle readings of capacitive bend sensors $\alpha_i$ can be converted into bending angles $\theta_i$ via assumption of piecewise constant curvature (PCC) \cite{toshimitsu_sopra_2021,wertz_trajectory_2022,patterson_robust_2022}.
If PCC holds, $\theta_i = \alpha_i/2$.
Horton's operation departs substantially from PCC.
However, we focus on proof-of-concept balancing in \textit{any} pose while leaving kinematic tracking for future work.

\subsection{System Architecture}


Controller calculations occur on a laptop (Core i5, 2.6 GHz, 16 GB RAM) attached to an Arduino Mega microcontroller (Fig. \ref{fig:horton_hardware_and_system_architecture}(d)).
In this framework, sensor readings include both pose states $\bq(k)$ and other internal states $\bx(k)$, which here are the robot's SMA temperatures.
At time $k$, a controller for the robot's pose calculates an input signal $\bu(k)$ intended to regulate the robot around a specified setpoint $\bar\bq(k)$, which is passed through the supervisor to get the safety-guaranteed PWM signal $\hat\bu(k)$, inducing heating current $\mathcal{I}$.
The following section details the signals and computations, including the proportional-integral controller with anti-windup (PIAW) for pose control.
This architecture is generalizable to other soft legged robots with actuator dynamics.


\section{Supervisory Control for Soft Legged Robots}

The robot presented here experiences unsafe operation of its actuators under disturbances.
A simple on/off hard stop on temperature causes chatter.
Instead, this article's supervisory control system dynamically saturates the input signal to maintain a provable safety specification.

\subsection{System Model}

Soft robots are highly nonlinear systems in general \cite{della_santina_control_survey}, and many state parameterizations exist for the dynamics of soft bodies and actuators \cite{goury_fast_2018,della_santina_model-based_2020,huang_dynamic_2020}.
The most common model-based representations for control \cite{armanini_soft_2021} result in a system of ordinary differential equations, which we assume are discretized appropriately:

\begin{equation}
    \bz(k+1) = f(\bz(k), \bu(k))
\end{equation}

\noindent with states $\bz \in \mathbb{R}^N$ and inputs $\bu \in \mathbb{R}^P$.

This article focuses on a subset of these states $\bx \in \mathbb{R}^M$ for which safety concerns exist, as in $\bz = [\bq^\top \; \bx^\top]^\top$.
In the Horton robot, the states $\bx = \begin{bmatrix} T_1 & \hdots & T_m \end{bmatrix}^\top$ are the SMA wire temperatures, which will readily exceed a safe limit under feedback architectures that do not account for disturbances.
We leave the body dynamics safety in $\bq$ for future work.

It has been established that the temperature dynamics of SMA wires are accurately approximated by thermal models for Joule heating \cite{mollaei2012optimal,katoch2015trajectory,bhargaw2013thermo,cheng2017modeling,wertz_trajectory_2022,sabelhaus_safe_2022}, which for this study takes the form of a scalar affine system:

\begin{equation}
    T_i(k+1) = -\frac{h_c A_c}{C_v}(T_i(k) - T_0)\Delta_t + \frac{1}{C_v} \Delta_t P_i(k)
\end{equation}

\noindent with input power represented in terms of a pulse-width modulation input $u_i(k) \in [0,1]$ as $P_i(k) = \rho J^2 u_i(k)$.
Here, $h_c$, $A_c$, $C_v$, $\rho$, and $J$ are various constants for material, geometry, and electrical properties, all fixed scalars.
Lumping all unknown constants produces

\begin{equation}
    T_i(k+1) = a_{(1,i)} T_i(k) + a_{(2,i)} u_i(k) + a_{(3,i)}.
\end{equation}

\noindent We assume that $\{a_1, a_2, a_3\}$ will be fit from hardware data on SMA temperature and duty cycle as in prior work \cite{wertz_trajectory_2022}.

Lastly, we transform this system of scalar affine equations into one linear system. 
Augmenting the states with a brief abuse of notation as $\bx_i = \begin{bmatrix} T_i & 1 \end{bmatrix}^\top \in \mathbb{R}^2$ gives each wire's dynamics as $\bx_i(k+1) = \bA_i \bx_i(k) + \bB_i u_i(k)$, where $\bA = [ a_{(1,i)}, \; a_{(3,i)}; \; 0, 1] \in \mathbb{R}^{2 \times 2}$ and $\bB = \begin{bmatrix} a_{(2,i)} & 0 \end{bmatrix}^\top \in \mathbb{R}^2$.
Arranging into one system $\bx(k+1) = \bA \bx(k) + \bB \bu(k)$ produces a block structure:

\begin{equation}\label{eqn:linsys_block}
    \bx(k+1) = \begin{bmatrix}
        \bA_1 & & \\
        & \ddots & \\
        & & \bA_m
    \end{bmatrix} \bx(k) + \begin{bmatrix}
        \bB_1 \\ \vdots \\ \bB_m
    \end{bmatrix} \bu(k).
\end{equation}

\noindent We use eqn. (\ref{eqn:linsys_block}) in most of the remainder of the article to drop the $i$-indexing.
\subsection{The Supervisor's Saturating Controller}

This work assumes the desired safety condition is pre-specified as a bound on $\bx$.
For Horton, bounds are a maximum temperature $T^{MAX}$, and we seek to guarantee that $\bx_i(k) \leq \begin{bmatrix} T^{MAX} & 1 \end{bmatrix}^\top$ $\forall k \in \mathbb{N}^+$, i.e., $\bx(k) \leq \bx^{MAX}$.
The following section adapts our recent result which does so \cite{sabelhaus_safe_2022}, now for a robot with multiple SMA-powered limbs (eqn. (\ref{eqn:linsys_block})).

Consider first the standard result \cite{baggio_data-driven_2019} that a linear system can be driven from $\bx(k)$ to a setpoint $\bx^{SET}$ in $K$-many steps, given reachability conditions, via the minimum-energy control sequence $\bu^*(k) = \bB^\top (\bA^\top)^{K-k-1} \bW_K^\dag (\bx^{SET} - \bA^K \bx(k))$, where $\bW_K$ is the $K$-step controllability Grammian \cite{baggio_data-driven_2019,sabelhaus_safe_2022}.
Taking a one-step ahead window $K=k+1$ as a worst case, with shortest time to reach $\bx^{SET}$,

\begin{equation}\label{eqn:onestep_set}
    \bu^*(k) = \bB^\top (\bB \bB)^\dag (\bx^{SET} - \bA \bx(k)).
\end{equation}

The system in eqn. (\ref{eqn:linsys_block}) is monotone \cite{sabelhaus_safe_2022}, so applying $\bu(k) \leq \bu^*(k)$ guarantees that $\bx(k+1) \leq \bx^{SET}$ elementwise.
Intuitively, applying less current generates lower wire temperatures.
However, to make this approach more robust, we propose to constrain inputs to some fraction of the setpoint amount, $\bu(k) \leq \gamma \bu^*(k)$, where $\gamma \in (0,1)$.
Our prior work showed that maintaining this condition in a closed loop changes the desired equilibrium point \cite{sabelhaus_safe_2022}, and so to obtain $\bx^{MAX}$ as the constraint, the following adjustment is required of the setpoint:

\begin{equation}\label{eqn:setpoint_adjust}
    \bx^{SET} = \left( 1 / \gamma \right) \left(\bI - \left(1 - \gamma \right) \bA \right) \bx^{MAX}.
\end{equation}

\noindent We combine eqns. (\ref{eqn:onestep_set})-(\ref{eqn:setpoint_adjust}) with feedback of $\bx(k)$,

\begin{align}\label{eqn:umax}
    \bu(\bx(k))^{MAX} & = \gamma \bB^\top (\bB \bB^\top)^\dag ( \left( 1 / \gamma \right) (\bI \; \hdots \notag \\
    & - \left(1 - \gamma \right) \bA)  ) \bx^{MAX} - \bA \bx(k) ).
\end{align}

The closed loop system obtained with $\bu(\bx(k)) = \bu(\bx(k))^{MAX}$ can be written in terms of the error $\be(k) \defeq \bx(k) - \bx^{MAX}$ as $\be(k+1) =\gamma \bA \be(k)$, which is exponentially stable under easily anticipated conditions.
However, that does not necessarily guarantee $\bx(k) \leq \bx^{MAX} \; \forall k$, which corresponds to $\be(k) \leq 0$.
To verify this safety property, we prove that the set of states $\mathcal{S} = \{\be \; | \; \be \leq 0\}$ is invariant under the closed-loop dynamics; i.e., $\be(0) \leq 0 \Rightarrow \be(k) \leq 0 \; \forall k \in \mathbb{N}^+$.
Noting that $\mathcal{S}$ is a polyhedron, it is invariant \cite{gilbert_linear_1991,blanchini_set_1999} under the transition matrix $\gamma \bA$ if $\mathcal{S} \subseteq \{ \be \; | \; \gamma \bA \be \leq 0\}$. 
That condition can also be written using set equivalence, 

\begin{equation}\label{eqn:S_invariance}
    \mathcal{S} \cap \{ \be \; | \; \gamma \bA \be \leq 0\} = \mathcal{S}.    
\end{equation}

Calculations from prior work \cite{sabelhaus_safe_2022} showed that eqn. (\ref{eqn:S_invariance}) held true for all experimentally-calibrated SMA thermal dynamics models, and therefore $\mathcal{S}$ is invariant.
We refer the reader to the literature for further discussions concerning the $Pre$ operation implied in eqn. (\ref{eqn:S_invariance}) and maximum invariant set calculations \cite{gilbert_linear_1991,blanchini_set_1999,sabelhaus_safe_2022}.

\subsection{Incorporating Pose Feedback and Verifying Safety}

Operating the robot using $\bu(\bx(k)) = \bu(\bx(k))^{MAX}$ maintains safe temperatures but does not allow for controlling the pose of the robot, i.e. the other states in $\bz$.
Instead, assume that there is another feedback controller $\bv(\bz(k))$ developed independently, which we would like to operate unless it would lead to unsafe states.
Propose the following (elementwise) composition, recalling that $\bx$ is part of the full state $\bz$,


\begin{equation}\label{eqn:awesome_controller}
    \hat u_i(\bz(k)) = \begin{cases}
        v_i(\bz(k)) \quad \; \; \; \text{if} &  v_i(\bz(k)) \leq u_i(\bx(k))^{MAX} \\
        u_i(\bx(k))^{MAX} & \text{else}
    \end{cases}
\end{equation}

The closed loop system under $\hat \bu(\bz(k))$ has the same invariance properties, and therefore safety, as the dynamics under solely $\bu(\bx(k))^{MAX}$. 
Formally:

\begin{theorem}
    The set $\mathcal{S} = \{\be \; | \; \be \leq 0 \}$ is invariant for the closed loop system obtained by applying $\hat\bu(\bz(k))$ to the dynamics of eqn. (\ref{eqn:linsys_block}) if invariance under $\bu(\bx(k))^{MAX}$ has been verified via eqn. (\ref{eqn:S_invariance}).
\end{theorem}

\begin{proof}
    Consider any state $\bx_i(k) \leq \bx_i^{MAX}$, i.e. $\be(k) \in \mathcal{S}$. Since $\mathcal{S}$ is invariant under the supervisor alone, $\bA_i \bx_i(k) + \bB_i \bu_i(\bx(k))^{MAX} \leq \bx_i^{MAX}$.
    Then since eqn. (\ref{eqn:linsys_block}) is a monotone control system and by definition of $\hat\bu$ in eqn. (\ref{eqn:awesome_controller}),

    \begin{align}
        \hat \bu_i(\bz(k)) & \leq \bu_i(\bx(k))^{MAX}, \notag \\
        \bA_i \bx_i(k) + \bB_i \hat \bu_i(\bz(k)) & \leq \bA_i \bx_i(k) + \bB_i \bu_i(\bx(k))^{MAX}, \notag \\
        \bA_i \bx_i(k) + \bB_i \hat \bu_i(\bz(k)) & \leq \bx_i^{MAX}, \notag \\
        \therefore \bx_i(k+1) & \leq \bx_i^{MAX}, \notag
    \end{align}

    \noindent so $\be(k+1) \in \mathcal{S}$, and $\mathcal{S}$ is invariant by induction.
\end{proof}

\textit{Remark}.
Our prior work also establishes standard considerations such as e.g. (Lipschitz) continuity \cite{sabelhaus_safe_2022}.

\begin{figure*}[tb]
    \centering
    \includegraphics[width=\textwidth,height=5.5cm]{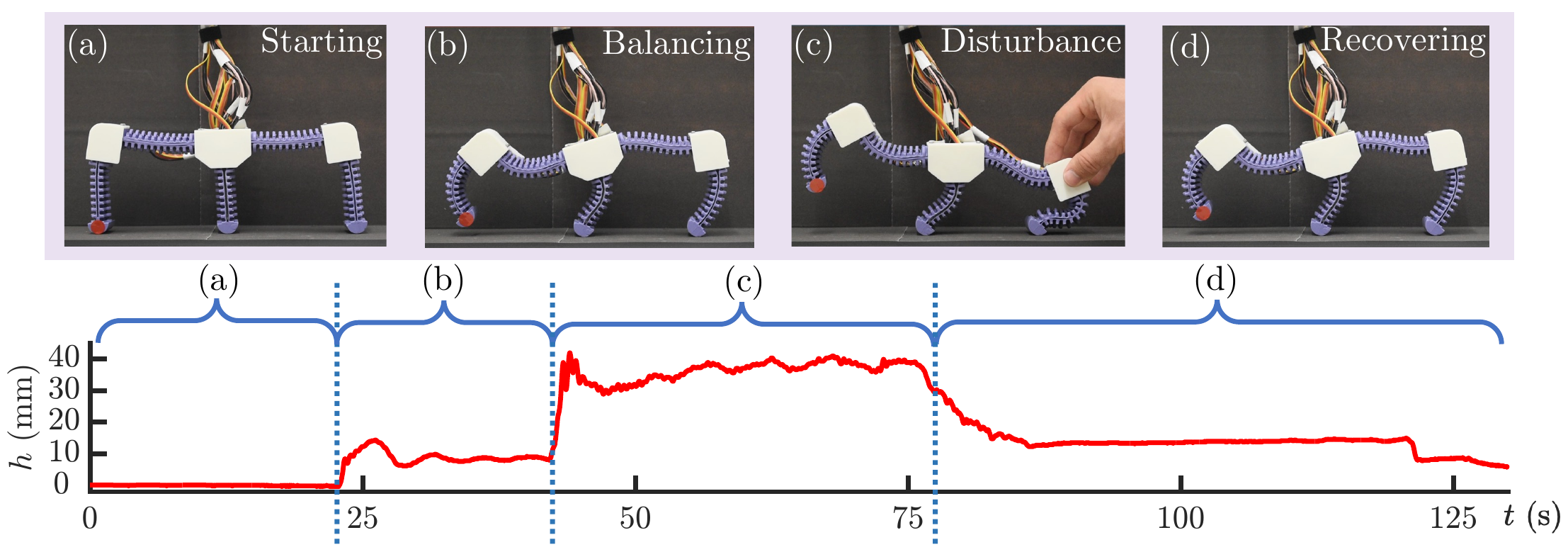}
    \caption{Feedback control test of Horton lifting its front leg. The test's four stages (a)-(d) demonstrate recovery from a human disturbance. Foot height (red dot distance from testbed surface) converges when the robot is undisturbed, showing dynamic balancing. Representative snapshots correspond to steady-state operation during each stage.}
    \label{fig:horton_foot_height}
    \vspace{-0.4cm}
\end{figure*}


\subsection{Pose Controller}

The preceding sections establish the safety of actuator states for any choice of pose controller $\bv(\bz(k))$.
For tests in this article, we use proportional-integral controllers with anti-windup (PIAW) to track a bending angle for each limb of the $J$-many limbs.
To do so, with limb angles as pose states $\bq = \begin{bmatrix} \theta_1 & \hdots & \theta_J \end{bmatrix}^\top$ and target angles as $\bar \bq$, define the error in each limb pose as $\delta_i(k) \defeq \theta_i(k) - \bar\theta_i(k)$.

For the control, group the $j$-th limb's two SMAs into one antagonistic pair as in prior work \cite{patterson_robust_2022,sabelhaus_safe_2022}, forming a single input $\mu_j \in [-1, 1]$ by mapping positive duty cycles to one actuator and negative to the other:

\vspace{-0.3cm}
\begin{align}
    v_{2i} & = \mu_j, \quad v_{2i+1} = 0 \quad \text{if} \; \mu_j \geq 0, \notag \\
    v_{2i} & = 0, \quad v_{2i+1} = -\mu_j \quad \text{if} \; \mu_j < 0. \notag
\end{align}

\noindent Then, dropping indexing, take $\mu$ as saturated output from a commanded signal $\eta$ as $\mu = \text{sat}(\eta)$.
The scalar PIAW controllers are then defined in terms of $\eta$, with sampling time $\Delta_t$, as

\vspace{-0.7cm}
\begin{align}
    \eta(\delta(k)) & = K_P \delta(k) + K_I \Delta_t \sum_{\tau=0}^{k-1} [\delta(\tau) \hdots \notag \\ 
    & + K_A (\mu(\tau-1) - \eta(\tau-1))]
\end{align}

\noindent The anti-windup term $(\mu-\eta)$ compensates for saturation.
All gains for experiments in this article were tuned as in our prior work \cite{sabelhaus_-situ_2022, sabelhaus_safe_2022}, and $\gamma=0.2$.
However, we used $K_A=0.0$ so although anti-windup was not operated in practice, it is implemented as a feature of our framework that was key to prior results \cite{patterson_robust_2022,sabelhaus_safe_2022}.




\vspace{0.1cm}
\section{Balancing with Temperature Supervisor}


We validate Horton's pose feedback and safety-guaranteed control framework together in a balancing test with a human disturbance.
This test mimics an unknown environmental contact that may occur in future locomotion.

\begin{figure*}[tb]
    \centering
    \includegraphics[width=\textwidth,height=7.4cm]{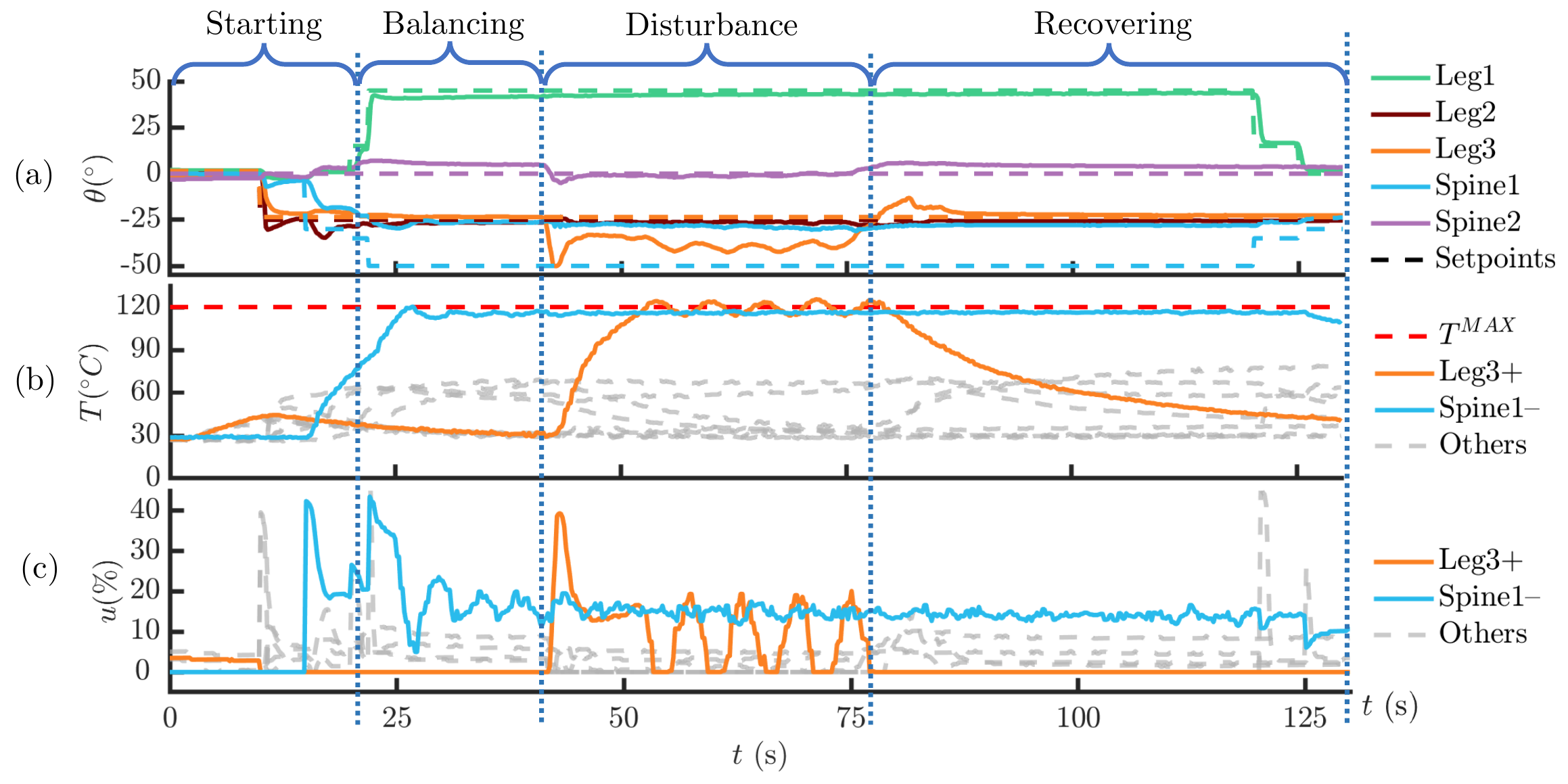}
    \caption{Data recorded during the balancing test from Fig. \ref{fig:horton_foot_height} include (a) the limb bending angles $\bq = [\theta_1, \; \hdots, \; \theta_5]^\top$ and corresponding setpoints, (b) the ten SMAs temperature readings $\bx$ and limit $T^{MAX}$, and (c) control inputs of PWM duty cycles $\hat\bu$ as a percent. The two highlighted SMAs in (b) and (c) show the action of the safety supervisor.}
    \label{fig:horton_state}
    \vspace{-0.4cm}
\end{figure*}


\subsection{Experimental Design}


We developed a trajectory of setpoint angles $\bar\bq$ through manual experimentation that corresponded with the front leg ($L_1$) lifting off the ground.
In the experiment, Horton was allowed to move unencumbered from its start position (Fig. \ref{fig:horton_foot_height}(a)), with a wait until it converged to one balancing pose (Fig. \ref{fig:horton_foot_height}(b)).
Then, a human disturbance (a hand pushing the robot down) intervenes and temporarily makes some setpoint angles infeasible (Fig. \ref{fig:horton_foot_height}(c)). 
Finally, the human disturbance is removed and the feedback controller recovers Horton to the balancing state (Fig. \ref{fig:horton_foot_height}(d)). 
For verification of balancing in addition to the convergence of our bending sensor readings, we attached a red dot sticker to Horton's front foot to track its height, with time series data obtained from post-processing a video recording.


\subsection{Results}

The feedback control test took 130 seconds in total to perform self-balancing, safety of actuator states under human disturbance, and recovery. 
Fig. \ref{fig:horton_foot_height} shows the height of the lifted front foot from the ground during the experiment. 
In the first balancing stage, the front foot's height converged to approximately 8.5 mm, with small oscillations due to noise.
The human's hand caused a large interference with leg $L_3$, raising the front foot up to 39.5 mm. 
Finally, the foot height returned to around 13.5 mm in the recovery stage, and was maintained for more than 30 seconds. 
The average balancing foot height of 11.9 mm is approximately 15\% of the robot's height, considered large versus past work on rigid robots \cite{park_finite-state_2013}.


The safety supervisor successfully constrained actuator states during the test, as shown in experimental data for bending angles $\bq = [\theta_1, \; \hdots, \; \theta_5]^\top$, SMA temperatures $\bx = [T_1, \; \hdots, \; T_{10}]^\top$, and control inputs $\hat\bu = [u_1, \; \hdots, \; u_{10}]$ (Fig. \ref{fig:horton_state}).
Limbs $L_1$ and $L_2$ converged quickly to their specified setpoint, as did spine $S_2$, albeit with some steady-state error as it was forced to bend when the robot's center of gravity shifted.
In contrast, spine $S_1$ was commanded to an infeasible pose: at $T^{MAX}$, its SMA wires could not provide enough force to fully counter the gravitational load of the front leg.
As a result, the safety supervisor activated for the $S_1$- actuator, which remained below $T^{MAX}$ throughout the test (Fig. \ref{fig:horton_state}(b), blue line).
Similarly, the human disturbance caused supervisor activation for the $L_3+$ actuator from 40-77 sec., which remained around $T^{MAX}$ with some chatter before the robot was released (Fig. \ref{fig:horton_state}(b), orange line).
This experiment demonstrates the effectiveness of our supervisory control method.

\section{Discussion and Conclusion}
This paper demonstrates for the first time a soft legged robot with full pose control and verifiable safety on aspects of its operation. 
Horton is capable of dynamic balancing on two of its three legs, and maintaining safe operation during disturbances. 
This progress toward dynamic, controlled soft legged locomotion could significantly expand robotic exploration of unknown environments.
The safety framework is applicable to many other legged soft robots with actuator dynamics.


The balancing test confirms the safety specification for Horton's $S_1$- spine actuator.
However, chatter in $L_3+$ shows some temperatures slightly above $T_{MAX}$.
This is expected, as the supervisor's predictions use the thermal dynamics model (eqn. (\ref{eqn:umax})) which was not recalibrated per-SMA.
The spines' wires are 25\% longer than the legs' wires, so imprecise model parameters are the likely cause, which does not indicate theoretical issues.

The balancing test showed the robot maintaining a lifted leg in otherwise-unstable poses.
However, though the pose controller converged to the same values of $\bq$ both before and after the disturbance (Fig. \ref{fig:horton_state}(a)), the foot height converged to a different value (Fig. \ref{fig:horton_foot_height}).
This is also expected, since the constant curvature assumption does not hold for Horton's limbs (c.f. Fig. \ref{fig:horton_balancing} vs. Fig. \ref{fig:horton_hardware_and_system_architecture}(d)), and confirms that our choice of state space does not uniquely determine its kinematics.
Future work will focus on alternative representations for $\bq$.
We emphasize this article's contribution as proof-of-concept for safety in any balancing motion.


Hardware designs, including the planar limbs and tether, were motivated by this article's focus on fundamental control results.
Future work seeks to design a Horton robot for practical locomotion, untethered, in 3D.
Design modifications will reduce weight and power needs by drawing from past work \cite{Huang2019}.
Stiffer materials (and additional SMA wires) could produce motions more consistent with the constant-curvature assumption.
This may allow safe control algorithms to be extended to the robot's pose in addition to its actuators.

\bibliographystyle{IEEEtran}
\bibliography{references}
\end{document}